\documentclass{article}
\usepackage[final,nonatbib]{nips_like}

\usepackage[T1]{fontenc}    
\usepackage{hyperref}       
\usepackage{url}            
\usepackage{booktabs}       
\usepackage{amsfonts}       
\usepackage{nicefrac}       
\usepackage{microtype}      

\usepackage{xcolor}
\hypersetup{
    colorlinks,
    linkcolor={red!0!black},
    citecolor={blue!80!black},
}

\usepackage{subcaption}

\usepackage{amsmath,amsfonts,amsthm,amsxtra,amssymb}
\usepackage{tikz}
\usetikzlibrary{bayesnet}

\usepackage{common}
\usepackage{macros}

\usepackage[nameinlink]{cleveref}
\usepackage{xifthen}

\DeclareMathOperator{\TC}{TC}

\newcommand{\E}{\mathbb{E}}

\renewcommand{\H}{\mathcal{H}}
\newcommand{\D}{\mathcal{D}}
\renewcommand{\L}{\mathcal{L}}

\newcommand{\x}{{y}}
\newcommand{\n}{{n}}
\newcommand{\y}{{z}}

\newcommand{\z}{\mathbf{x}}

\newcommand{\hz}{\hat{\z}}

\newcommand{\half}{\frac{1}{2}}

\newcommand{\be}{\begin{equation}}
\newcommand{\ee}{\end{equation}}
\newcommand{\bea}{\begin{eqnarray}}
\newcommand{\eea}{\end{eqnarray}}

\newcommand{\utext}[2]{\underbrace{#1}_{\text{#2}}}

\usepackage{enumitem}

\begin{document}

\title{A Separation Principle for Control\\ in the Age of Deep Learning}

\author{Alessandro Achille \& Stefano Soatto \\
Department of Computer Science \\
University of California, Los Angeles \\
405 Hilgard Ave, Los Angeles, 90095, CA, USA \\
\texttt{\{achille,soatto\}@cs.ucla.edu}
}

\maketitle

\begin{abstract}
We review the problem of defining and inferring a ``state'' for a control system based on complex, high-dimensional, highly uncertain measurement streams such as videos. Such a state, or representation, should contain all and only the information needed for control, and discount nuisance variability in the data.  It should also have finite complexity, ideally modulated depending on available resources. This representation is what we want to store in memory in lieu of the data, as it ``separates'' the control task from the measurement process. For the trivial case with no dynamics, a representation can be inferred by minimizing the Information Bottleneck Lagrangian in a function class realized by deep neural networks. The resulting representation has much higher dimension than the data, already in the millions, but it is smaller in the sense of information content, retaining only what is needed for the task. This process also yields representations that are invariant to nuisance factors and having maximally independent components. We extend these ideas to the dynamic case, where the representation is the posterior density of the task variable given the measurements up to the current time, which is in general much simpler than the prediction density maintained by the classical Bayesian filter. Again this can be finitely-parametrized using a deep neural network, and already some applications are beginning to emerge. No  explicit assumption of Markovianity is needed; instead, complexity trades off approximation of an optimal representation, including the degree of Markovianity.
\end{abstract}

\def\comment#1{ {\color{red}#1}}
\def\cut#1{{}}

\tableofcontents

\newcommand{\xzyfig}[2][]{{
    \if\relax\detokenize{#1}\relax \def\name{#2} \else \def\name{#1} \fi
    \if#21
        \node[latent] (z#2) {$z_{\name}$};
    \else
        \node[latent, right = of z\the\numexpr(#2-1)] (z#2) {$z_{\name}$};
        \edge[] {z\the\numexpr(#2-1)}{z#2}
    \fi 
    \node[obs, above = of z#2] (x#2) {$x_{\name}$};
    \node[latent, below = of z#2] (y#2) {$y_{\name}$};
    \if#21\else\edge[dashed] {x\the\numexpr(#2-1)}{x#2}\fi 
    \edge[] {x#2}{z#2}
    \edge[] {z#2}{y#2}
}}
\newcommand{\zxfig}[2][]{{
    \if\relax\detokenize{#1}\relax \def\name{#2} \else \def\name{#1} \fi
    \if#21
        \node[latent] (z#2) {$z_{\name}$};
    \else
        \node[latent, right = of z\the\numexpr(#2-1)] (z#2) {$z_{\name}$};
        \edge[] {z\the\numexpr(#2-1)}{z#2}
    \fi 
    \node[obs, below = of z#2] (x#2) {$x_{\name}$};
    \edge[] {z#2}{x#2}
}}
\def\dotsfig#1{{
    \if#11
        \node[] (z#1) {\ldots};
    \else
        \node[right = of z\the\numexpr(#1-1)] (z#1) {\ldots};
        \edge[] {z\the\numexpr(#1-1)}{z#1}

    \fi 
    \node[above = of z#1] (x#1) {\vphantom{$x_#1$}\ldots};
    \edge[dashed] {x\the\numexpr(#1-1)}{x#1}
}}
\def\twodotsfig#1{{
    \if#11
        \node[] (z#1) {$\ldots$};
    \else
        \node[right = of z\the\numexpr(#1-1)] (z#1) {$\ldots$};
        \edge[] {z\the\numexpr(#1-1)}{z#1}
    \fi 
}}


\def\x{{\color{red} y}} 
\def\y{{\color{red} z}} 
\def\z{{\color{red} x}} 
\def\u{{\color{red} u}} 

\def\x{{y}} 
\def\y{{z}} 
\def\z{{x}} 
\def\u{{u}} 

\def\Cov{Cov}

\def\D{{\cal D}} 
\newcommand{\indep}{\mathbin{\rotatebox[origin=c]{90}{$\models$}}}

\section{Introduction}

Say you have a time series of data and wish to store a function of it having constant complexity, which we call a representation, that is {\em useful} for a prediction or control task. ``Useful'' means that the representation retains all the information the data contain about the task.\footnote{For such a representation to exist, we must assume that the data satisfies a Markov model, an assumption on which we will come back later.} 

For example, if $\x_t$ is the output of a linear time-invariant system with ``true'' finite-dimensional state $\z_t$, driven by white zero-mean Gaussian noise, then a sufficient representation $\hat \z_t$ of the data $\x^t := \{\x_0, \dots, \x_t\}$ for the task $\y_t$, for instance, prediction $\y_t = \x_{t+1}$, is the mean and covariance of the posterior density $p(\z_t | \x^t)$ \cite{kalman60}, or equivalently the posterior itself (more on this equivalence later). This representation summarizes all past history of the data for the purpose of the task (in this case, predicting its future\footnote{If the model is not known, the representation includes a constant component that belongs to an equivalence class of realizations \cite{arun1990balanced}, and there is an elegant geometry that is exploited in subspace system identification \cite{lindquist1979stochastic}. Even if there is no ``true'' finite-dimensional state, under the Markovian assumption with Gaussian inputs of known dimension, one can infer a finite-dimensional predictor along with the state and model parameters \cite{akaike1974new}.}). In other words, given the representation, past data is independent of future data. Such independence makes it possible \emph{separate} inference of the state given the measurements, from control design given the estimated state \cite{lewis2012optimal}. 

Such a {\em separation principle} has served the practitioner well over the years, but has left us with few tools for when the underlying assumptions are not satisfied: What if noise is not additive or Gaussian? What if the ``true state'' is high- or even infinite-dimensional? What if the data is also high-dimensional and almost all of it irrelevant for the control task? 

Unfortunately, these are conditions the practitioner of Robotics and Autonomous Systems faces routinely: The task may include navigation in an unknown environment populated by objects whose shape is described by (infinite-dimensional) surfaces and reflectance functions; the data may include images with millions of channels (pixels), predicting most of which is irrelevant to the navigation task; nuisance factors may include occlusion, changes of illumination and pose, which are far from white, zero-mean Gaussian ``noise.'' Does a separation principle exist for this kind of scenario? Is it still possible to infer a bounded-complexity function of the data, that can be stored in memory in lieu of the data, with no information loss? 

To be sure, there have been many attempts to answer these questions. In {\em sufficient dimensionality reduction} \cite{chiaromonte2002sufficient,shyr2010sufficient} one aims to identify small-dimensional statistics ({\em e.g.,} projections) that summarize the data. Similarly, the classical use of invariants in image analysis is to remove redundancy from the data by mapping it to the quotient space, which for image can be a thin set \cite{sundaramoorthiPVS09}. 
These approaches have had limited impact in Robotics and Autonomous Systems. Sufficient reductions are either to restrictive (linear projections) or too hard to compute and difficult to use.
But what if we go the opposite way of dimensionality reduction? What if we instead increase the dimensionality beyond that of the data, already in the millions? What if, instead of computing statistics (deterministic functions of the data), we allow representations to be arbitrary stochastic functions?

At least for simpler tasks, such as classification \cite{krizhevsky2012imagenet} or control in a finite setting, deep neural networks (DNNs) with hundreds of millions of parameters have shown remarkable empirical success. Can we leverage this success to infer representations of time series specifically for filtering and control tasks? Is there a theoretical framework that explains why large networks would work well for control?

Using a large network seems ill-advised at first: The bias-variance dilemma \cite{cover2012elements} states that as we increase the complexity of a model inferred from finitely sampled data, its ability to capture the underlying distribution degrades, a phenomenon referred to as {\em overfitting} that seems at odd with DNN's empirical success \cite{zhang2016understanding}. However, if the network complexity is measured by information content, rather than dimension, a well-trained DNN for classification faithfully obeys the bias-variance tradeoff \cite{achille2017emergence}, and relaxing representations to be stochastic functions has the double advantage of simplifying the computation of information quantities and analyzing the properties of the resulting representations.

In this paper, we study representations for Robotics and Autonomous Systems, using tools from statistics and information theory, and deep neural networks as the class of functions implementing them (realizations). We first review the simple case of a model with trivial dynamics, to introduce the machinery of deep neural networks, and then extend it to the dynamic case. 

\subsection{Outline of the Paper}

In \Cref{sec:defining}, we introduce the defining properties of representations and formalize the notions of sufficiency, minimality, invariance and separation. Since representations that are minimal and sufficient do not exist in general in finite dimensions \cite{jeffreys1960extension}, we start from the posterior, which is minimal-sufficient \cite{bahadur1954sufficiency} but infinite-dimensional, and frame the problem of learning representations as an approximation problem where complexity is modulated explicitly in the Information Bottleneck Lagrangian (IBL). This is a cost functional to be minimized with respect to a class of functions in a sufficiently rich set.

Deep neural networks (DNNs) are universal approximants in the limit where the number of parameters goes to infinity, like many other function classes. However, they enjoy a peculiar coupling between the model parameters and the properties of the learned representation that make them better than most when we want the representation to be invariant to nuisance variability, and having components that are maximally independent (``disentanglement''). In \Cref{sec:dnn} we give a succinct introduction to deep neural networks (DNNs), to the extent needed to follow the rest of the paper. 

In \Cref{sec:core-results} we present a series of core results that explain in what sense deep networks are approximations of optimal representations for static systems \cite{achille2016information}. Specifically, through the use of the IB Lagrangian, we formalize the trade-off between the complexity of the data representation, and the error we commit when we use this representation to solve the task in lieu of the original data.

However, at first sight the IB Lagrangian does not address important properties of a representation, invariance and disentanglement, that we should then deal with separately. Nonetheless, we show that, given sufficiency, invariance is equivalent to minimality. We also show that the IBL is equivalent to the cross-entropy loss typically used for classification tasks in deep learning except for an added regularizer, thus creating and important link between the (information-theoretic) optimal representations and the Deep Learning practice. Furthermore, it has been shown that some heuristic methods in used in optimizing deep networks (stochastic gradient descent, Dropout and its variants) approximate this regularizer \cite{achille2017emergence}. We then show that, somewhat counter-intuitively, stacking layers of neural networks increases minimality of the representation, and therefore invariance. This is tied to the architecture of deep networks and partly explains their success. 

Architecture design is also critical in coupling the optimization process where the IBL is minimized with respect to the parameters of the network (weights), with the desirable properties of the resulting representations (activations), which we outline in \Cref{sec:defining}. Specifically, in \Cref{sec:duality} we describe an inequality that links the activations of a deep network (a representation of the test datum) and its weights (a representation of the training set). This duality also sheds light on the generalization properties of DNNs. 

In \Cref{sec:recurrent} we extend the model to dynamical systems. By explicitly introducing a
task variable, which is in general separate from the data, we open the possibility for drastically more efficient representations than those sufficient for future data prediction, while still learning end-to-end with a simple filter.

In \Cref{sec:discussion} we discuss some properties and limitations of the representation proposed. Specifically we discuss the limitations of the Markovianity assumptions in classical models, and how the proposed model partially overcomes it by trading it off against complexity costs.

\subsection{Related work} 

Deep Learning is impacting many areas of engineering and science, including time series forecasting \cite{koren2010collaborative}, showing promise especially when some of the most common hypotheses underlying conventional methods, such as Markovianity, are not satisfied \cite{wu2017recurrent}. Several have studied extensions of classical Bayesian filtering, including using neural networks, but the resulting approaches had drawbacks: (a) the complexity of the update rule, which in the classical Bayesian setting requires computing the posterior of the data given the hidden state, is problematic for high-dimensional data types such as images; (b) the only task considered is prediction of the data, which is usually overkill when the actual task is, say, control: One does not need to model the complex reflectance properties of the world to decide whether to steer a vehicle to the right. Finally,  (c) existing methods do not allow an explicit trade-off between the complexity of the hidden state and the quality of the prediction error. Notice, however, that Variational Bayesian methods can be used as a partial solution to the first problem \cite{krishnan2015deep,raiko2009variational}, as long as the function class covers the underlying data distribution, which remains an open problem for natural images.

Among other methods, the Deep Kalman filter \cite{krishnan2015deep} assumes the existence of a Gaussian latent state and non-linear transformations that explain the observations and use a variational auto-encoder (VAE) to infer such a model. This choice is restrictive as the only task allowed is the reconstruction of the measurements. Also the method focuses on batch system identification whereas we are interested in a causal, on-line scheme.

More directly related to our approach, \cite{langford2009learning} suggest that, rather than finding a (generally complex) hidden state, we could focus on finding a statistic $\z_t$ of the past data $\x^t$ that separates past data from future predictions. Such a statistic can be learned and updated without using Bayes' rule, thereby avoiding the complex computations of the data posterior $p(\x_t|\z_t)$. However, their analysis is restricted to linear models, and the task to the prediction of future data.

Constructing deterministic functions of the data (statistics) that separate the past from the future requires $N+N^2$ (embedding) dimensions (the mean and covariance maintained by a Kalman filter \cite{jazwinski2007stochastic}, where $N$ is the dimension of the state-space); however, as we will show, a stochastic representation can make do with $N$  dimensions, at the cost of maintaining samples from the distribution, as in a particle filter. Sigma-point filters \cite{wan2000unscented} are deterministic sample-based representations that falls in between these two cases.  In this case as well, the task is prediction of the measurements. We allow the task to be more general, including the case whereby one does not care to be able to reproduce every channel of the measurements (e.g. the color of every pixel in the image) but instead cares only for a small projection or quotient of the data with respect to the action of nuisances. Our model also allows even more flexibility relative to the strong assumption of Markovianity implicit in the classical filtering equations. The inferred state can be thought as a separator but only for the measurements, as opposed to a more general task, which makes the problem less tractable than in our case, as the statistics that matter for control typically have far smaller complexity than the data. Also the proofs provided only apply to minimal realizations. In our model we trade-off Markovianity and complexity, which is not contemplated in the classical filtering equations.

The trade-off we seek between complexity of the representation and sufficiency for future prediction is also closely related to the Minimum-Information LQG Control of \cite{fox2016minimum}, which explicitly accounts for the agent having a representation of bounded complexity.  \cite{fox2016minimum} only addresses the linear-quadratic Gaussian case, but gives a complete account of that. Similarly \cite{tiomkin2017control,rubin2012trading} deal with capacity costs, but in continuous time. The general principles are laid out in \cite{fox2016principled}.

The theory we describe here emphasize that, in order to obtain efficient representations of the data, we should focus on a specific task, such as control, rather than predicting high-dimensional future data. \cite{dosovitskiy2016learning} assume that there is a low-dimensional vector of ``measurements'' separate from the actual measured data, which can be easily obtained, and on which the control loss depends linearly. In our parlance, predicting the future  is a task sufficient for control, and therefore allows us to learn a sufficient representation for control. In particular, optimal control reduces to minimum-prediction error, and one can simply train a network to predict future measurements conditioned on the current policy a given control action taken at present. Using this technique for control shows state of the art performance on video games. 

Finally, we study the quantity of information that the observed data contain about the parameters of the system plays. This quantity is considered in \cite{houthooft2016variational}, which uses a variational approximation similar to ours to measure the information content in the data using a neural network, and then learns a control policy for exploration that maximizes this information quantity. In this case the model is a constant parameter, akin to an assumption of time-invariance and equivalent to Markovianity.

\subsection{Preliminaries}

We denote the history of a process from time $t_0$ to $t$ by $\x_{t_0}^{t_1} = \set{\x_{t_0+1},\ldots,\x_{t_1}}$, where we omit the subscript when $t_0 = 0$. Thus, $\x^t$ denotes the measured data up to time $t$, while $\y_t$ denote the quantity of interest (task) at that time, which could be the value of the measurements at a future time, $\y_t = \x_{t+\tau}$. We will consider trivial dynamics at first, with each $(\x_t, \y_t) \sim p_\theta(\x,\y)$ independent and identically distributed (i.i.d.) from an unknown density $p_\theta$.

For random variables $\x, \y, \z$ we denote the expectation of $\x$ with respect to the measure $p(\x)$ by $\E_p(\x)$, Shannon's entropy by $H(\x)=\E_p[-\log p(\x)]$, conditional entropy by $H(\x|\y)=H(\x,\y) - H(\y)$, (conditional) mutual information $I(\x;\y|\z) = H(\x|\z) - H(\x|\y,\z)$, Kullback-Liebler's (KL) divergence $KL(p(\x) || q(\x)) = \E_p[\log p/q]$, cross-entropy $H_{p,q}(\x) = \E_p[-\log q(\x)]$, and total correlation by $TC(\z)$, defined as 
\[ \textstyle \TC(\z) = \KL{p(\z)}{\prod_i p(\z_i)},\]
where $p(\z_i)$ are the marginal distributions of the components of $\z$; $\TC(\z)$ is zero if and only if the components of $\z$ are independent, in which case we say that $z$ is \emph{disentangled}.
We make use of the following identity:
\begin{equation*}
\label{eq:information-kl}
I(\z;\x) = \E_{\x\sim p(\x)} \KL{p(\z|\x)}{p(\z)}.
\end{equation*}
We say that $\x, \z, \y$ form a Markov chain, indicated with $\y \to \z \to \x$, if $p(\y|\x,\z) = p(\y|\z)$. The Data Processing Inequality (DPI) for a Markov chain $\x \rightarrow \z \rightarrow \y$ ensures that $I(\x;\z) \ge I(\x; \y)$.

We define a \textbf{nuisance} to be any random variable that affects the observed data $\x$, but is not related to the task, $\y \indep \n$, or equivalently  $I(\y;\n)=0$.
Similarly, we say that the representation $\z$ is \textbf{invariant} to the nuisance $\n$ if $\z \indep \n$, or $I(\z;\n)=0$. When $\z$ is not strictly invariant but minimizes $I(\z;\n)$ among all sufficient representations, we say it is \textbf{maximally insensitive} to $\n$. It can be shown \cite{achille2017emergence} that the data can always be written as a function of  the task and of all nuisances affecting it.
Specifically, given a joint distribution $p(\x,\y)$ where $\y$ is a discrete random variable, we can always find a random variable $\n$ independent of $\y$ such that $\x=f(\y,\n)$, for some deterministic function $f$.

Given random variables $\x$, $\y$, and $\z$ with joint density $p(\x,\y,\z)$, we say that $\z$ is \textbf{sufficient} of $\x$ for the \textbf{task} $\y$ if we have the Markov chain $\x \to \z \to \y$, \textit{i.e.}, if $p(\y|\x,\z) = p(\y|\x)$. We say instead that the \textbf{posterior of $\z$, $p(\z | \x)$ is sufficient} of $\x$ for $\y$ if $p(\y|\x) = \int p(\y|\z)p(\z|\x)d\z$. Notice that the posterior of a sufficient representation is in turn always sufficient, since $p(\y|\x) = \int p(\y,\z|\x) d\z = \int p(\y|\z,\x)p(\z|\x)d\z = \int p(\y|\z)p(\z|\x)d\z$. The converse does not hold in general, but holds in the important case in which $\z$ is a deterministic function of $\x$. In this work we will focus on  the general case of  sufficient posteriors and abuse the notation to refer to both the random variable $\z$ and the posterior $p(\z|\x)$ as being sufficient.%
\footnote{An equivalent characterization using conditional expectations is to say that $\z$ is sufficient of $\x$ for $\y$ if $\E[f(\y)|\x,\z] = \E[f(\y)|\z]$ for any measurable function $f$, and similarly  the posterior is sufficient if $\E[\E[f(\y)|\z]|\x] = \E[f(\y)|\x]$ for any $f$.}

\section{Desiderata for representations}
\label{sec:defining}

We call a representation $\z$ of the data $\x$ any stochastic function $\z \sim p(\z|\x)$ of $\x$. Ideally, we would like $\z$ to be \textbf{(a)} \emph{sufficient} for the task $\y$, that is, all the information that $\x$ contains about the task should also be contained in $\z$, or $I(\z;\y) = I(\x;\y)$. In order to not squander resources, the representation $\z$ should also be \textbf{(b)} \emph{minimal}, that is $I(\x;\z)$ is smallest among all sufficient representations $\z$. Note that we are defining ``small'' in terms of information content, not dimension, of $\z$. Moreover, we would like it to be \textbf{(c)} \emph{invariant} to a nuisance $\n$, $I(\z;\n) = 0$ or, if that is not possible, at least \emph{maximally insensitive} to it, {\em i.e.,} $I(\z;\n)$  is minimized. Note that we require invariants to be uninformative, not constant, with respect to nuisance variability. We impose no requirement on identifiability and harbor no hope of uniqueness of representations. However, to facilitate their use, we do wish for the components of $\z$ to be \textbf{(d)} \emph{maximally disentangled}, that is, we want $\TC(\z) = \KL{p(\z)}{\textstyle \prod_i p(\z_i)}$ to be minimized.

The first two properties are satisfied by any {\em minimal sufficient} representation, which can be found by solving 
\begin{align*} 
    \text{min}_{p(\z|\x)}& \quad I(\x;\z) \\
    \text{s.t.}& \quad H(\y|\z) = H(\y|\x)
\end{align*}
or minimizing the corresponding \emph{Information Bottleneck Lagrangian} (IBL) \cite{tishby2000information}: 
\begin{equation}
\L = \utext{H(\y|\z)}{cross-entropy} + \  \beta\utext{I(\z;\x).}{regularizer}
\label{eq:IBL}
\end{equation}
The IBL trades off \emph{sufficiency} and \emph{minimality}, regulated by $\beta$ and can be optimized efficiently when the $\z$ is parametrized by a neural network \cite{achille2016information,alemi2016deep}. However, we are also interested in the other properties, invariance and disentanglement, that are not explicit in the IBL and are the focus of the next section.

\section{Learning invariant and disentangled representations}
\label{sec:core-results}

We now present a key result connecting minimality of a representation and invariance to nuisances \cite{achille2017emergence}.

\begin{prop}
Let $\n$ be a nuisance affecting the data $\x$. Then, for any representation $\z$ of $\x$ we have:
\[
\utext{I(\z;\n)}{invariance} \leq \utext{I(\z;\x)}{minimality} - \utext{I(\x;\y)}{constant},
\]
where the RHS is minimized when $\z$ is minimal. Moreover, there always exists a particular nuisance $\n$ such that
equality holds up to a (generally small) residual $\epsilon$, that is
\[
I(\z;\n) = I(\z;\x) - I(\x;\y) - \epsilon,
\]
where  $\epsilon := I(\z;\y|\n) - I(\x;\y)$. In particular
$0 \leq \epsilon \leq H(\y|\x)$,%
\footnote{Notice that $\epsilon\leq H(\y|\x)$, and usually $H(\y|\x) \ll I(\x;\z)$, so we can generally ignore the extra term.}
and $\epsilon=0$ whenever $\z$ is a deterministic function of $\x$.  Under these conditions, a sufficient statistic $\z$ is invariant (maximally insensitive) to nuisances if and only if it is minimal.
\end{prop}

This result implies that, rather than manually imposing invariance to nuisances in the representation, which is usually difficult, we can construct invariants by simply reducing the amount of information that $\z$ contains about $\x$, while retaining sufficiency. As we discussed, this can be done minimizing the IB Lagrangian using a neural network \cite{achille2016information}.

While we will analyze deep networks in \Cref{sec:dnn},
this result, together with the Data Processing Inequality, already suggests an advantage in stacking multiple intermediate representations into ``layers.'' In fact, suppose suppose that we have a Markov chain of representations
\[
\x \to \z_1 \to \z_2 ,
\]
such that there is an information bottleneck between $\z_2$ and $\z_1$, that is, $I(\z_1;\z_2) < I(\z_1;\x)$. Then, if $\z_2$ is still sufficient, then it is necessarily more minimal, and therefore more invariant to nuisances, than $\z_1$. Notice that bottlenecks are easy to create, either by reducing the dimension so that $\dim(\z_2) < \dim(\z_1)$, or by introducing noise between the $\z_2$ and $\z_1$. This is indeed common practice in designing and training deep networks, which concatenate multiple layers 
\[
\x \to \z_1 \to \z_2 \to \ldots \to \z_L,
\]
so that, whenever layer $\z_L$ is sufficient of $\x$ for $\y$ (which is imposed by the training loss), then $\z_L$ is more insensitive to nuisances than all the preceding layers.

This also relates to the notion of Actionable Information \cite{soatto2013actionable}, which is ${\cal H}(\x) := H(f(\x))$: In the special case when  $\z = f(\x)$ is deterministic, a representation that minimizes the IB Lagrangian also maximizes Actionable Information \cite{achille2017emergence}.

Finally, it has been shown in \cite{achille2016information} that, when assuming a factorized prior in the activations, the IBL also bounds Total Correlation, so minimizing it yields a representation that trades off sufficiency with complexity, invariance, and disentanglement.

\section{Learning with Deep Neural Networks}
\label{sec:dnn}

In this section we sketch the very basics of deep learning, first by describing the class of functions realized by deep neural networks, and then the choice of functionals and optimization schemes used for determining their parameters. In the following section we show how this process, despite being agnostic of desirable properties of the representations outlined in the previous sections, manage to achieve just that by exploiting a peculiar information duality between the weights and the activations of the network.

\subsection{Function class of DNNs}

A Deep neural network (DNN) is a parametrized class of non-linear functions obtained by composing multiple \textit{layers}: Each layer implements a linear transformation of its input, which is the output of the previous layer, followed by a (generally element-wise) non-linearity. Specifically, let $\x := \z^0 \in \R^d$ be the input data, and let $W^k \in \R^{d_{i-1}\times d_i}$ be a matrix, where $d_0:=d$. Then, we define the ``activations'' (output) of the $k$-th layer as $\z^{k} = \phi_k(W^k \z^{k-1})$, where $\phi_k$ is a nonlinear function. A common choice for the non-linearity is $\phi_k(x) = \max(0,x)$, also called a Rectified Linear Unit (ReLU). The output $\z^K$  of a network with $K$ layers is the function
\[F(\x;w) = \phi_K(W^K \phi_{K-1} (W^{K-1} \ldots \phi_1(W^1 \x)\ldots)),\]
where $w=\set{W^1,\ldots,W^k}$ is the set of parameters, or \emph{weights}, of the network. Each $\z^k$ can be considered as a representation of the original input $\x$, and its component are generally called \emph{features} (or feature maps, or activations, or responses). By the data processing inequality, $\z^k$ contain no more information than $\x$; however, as we will see, in a well trained network we expect $\z^k$ to contain all (and only) the information necessary for the task. Since the output of the network is often a (conditional) probability distribution (\textit{e.g.}, the probability $p(\y|\x)$ of a label $\y$ given the image $\x$), the last non-linearity is usually a Softmax non-linearity, $\mathrm{softmax}(x)_i = e^{\x_i}/(\sum_j e^{\x^j} )$, which ensures that the output of the network is positive and sums to one.

When the input $\x$ has some particular structure, such as an image, the linear transformation $W^k$ can be chosen to exploit this structure. For example, when $\x$ is an image, it is common to choose $W^k$ to be a set of convolutions. Networks using convolutional maps, known as \emph{convolutional neural networks} (CNNs), have the notable property that their features are invariant to translations \cite{lecun1990handwritten}, and have considerably fewer parameters (the number of parameters depends only on the size of the filters, which are generally small, rather than on the size of the image). Aside from reducing the size of the parameter space, the use of convolutions has a drastic, and not yet fully understood, effect in achieving desirable properties of the networks when operating on imaging data \cite{soatto2016visual}.

\subsection{Loss function and optimization}

The output of a network is usually interpreted as a probability distribution $q(\y|\x,w)$ over the inference target $\y$ (\textit{e.g.}, the label of an image, the position of an object). Per \cite{soatto2016visual}, if that approached the true posterior, it would be a minimal sufficient representation.

When $\y$ is a discrete random variable, this identification can be done directly by letting the output $F(\x,w)$ of the network be a probability vector (or an un-normalized likelihood function). When $\y$ is continuous we can chose a family of parametrized distributions, and let the network output the  parameters (\textit{e.g.}, mean and variance for a normal distribution). In both cases, we will think of a deep network as a map $\x \mapsto F_w(\x) := q(\,{\cdot}\, |\x, w)$ where, absent any system dynamics, the parameters $w$ are constant and usually determined by maximizing the log-likelihood of the observed data, which leads to the cross-entropy loss
\[
L(w) = H_{p,q}(\y | \x, w) = \frac{1}{t} \sum_{i=1}^t -\log q(\y_i|\x_i,w).
\]
Notice that the cross-entropy loss can be decomposed as
\[
H_{p,q}(\y | \x, w) = H_p(\y|\x) + \KL{p(\y|\x)}{q(\y|\x,w)}.
\]
Since all terms are positive, and only the KL divergence depends on $w$, we can conclude that $L(w)$ is minimized if and only if $q(\y_i|\x_i,w)=p(\y_i|\x_i)$ on all observed samples, giving an alternative justification for the use of this loss. Moreover, in the special case where $q(\y|\x,w)$ is a normal distribution with fixed variance, the cross-entropy reduces to the usual $L_2$ loss for regression.

Minimization of the loss $L(w)$, and hence determining the weights $w$, is usually done using \emph{stochastic gradient descent} (SGD): We start by randomly initializing the parameters $w$ \cite{glorot2010understanding}. Then, at each step $k$, a random subset (\emph{``mini-batch''}) $(\x_{i_k}^{i_k+b},\y_{i_k}^{i_k+b})$ of size $b$, with $i_k \sim \mathrm{unif}(0,t-b)$, is sampled from the observed data $(\x^t,\y^t)$ and we compute the gradient $g^k$ relative to the mini-batch
\[
g^k = \frac{1}{b} \nabla_w H_{p,q}(\x_{i_k}^{i_k+b},\y_{i_k}^{i_k+b}) = \frac{1}{b} \sum_{j=0}^b \nabla_w \log q(\y_{i_k+j}|\x_{i_k+j},w)
\]
Since $\nabla_w L(w) = \E[g^k]$, we can see $g^k$ as an unbiased (but high-variance, or ``noisy'') estimate of the real gradient of the original loss function with respect to $w$. This can be computed efficiently since it requires computing the gradients on only $b$ samples rather than the whole collection of observed data, which can number in the millions. The weights are now updated using $w \leftarrow w + \eta_k \, g^k$, where $\eta_k > 0$ is called the \emph{learning rate}. It is known \cite{nesterov2013introductory} that when the loss function is strongly convex, the gradients are Lipschitz, and the learning rate decreases as $\eta_k = 1/k$, then SGD converges to the global optimum of the loss with convergence rate $O(1/t)$.

There are two main challenges one faces in carrying out this optimization: (i) the loss function is highly non-convex, therefore SGD can get stuck in a sub-optimal local minimum, and (ii) even if a global minimum is found, the parameters could be overfitting the data, meaning that while $w$ minimizes the loss on the observed data, the loss evaluated on unseen (future) data could be much larger.

The first problem (i) is partly addressed by SGD itself: Because of the noise added in the computation of the gradient by SGD, the optimization typically settles on extrema that are close to the global minimum in value. Variants of SGD include using Nesterov's momentum \cite{nesterov2013introductory}, which generally yields faster training and improved performance of the network. Other algorithms, like RMSProp and Adam \cite{kingma2014adam}, use the gradient history to reduce the variance in the estimate of the gradient, which is also adapted to the local geometry of the loss function. While in some tasks, such as stochastic optimal control (``Reinforcement Learning'') \cite{mnih2015human}, these algorithms show drastically improved performances as expected, on image classification and similar tasks the simpler SGD with momentum can still outperform them, suggesting that the noise added by SGD plays an important, positive, role in the optimization. There is at present a considerable amount of activity, but a dearth of results, in characterizing the topological and geometric properties of the loss function and designing algorithms that can exploit it to converge to minima that yield good generalization performance, as we discuss in Sect. \ref{sec:flatness}. Generalization, or lack thereof (``overfitting'') is the second problem (ii) which we discuss in more detail next.

\section{Duality and generalization}
\label{sec:duality}

One of the main problems in optimizing a DNN is that the cross-entropy loss in notoriously prone to overfitting: The loss is small for (past) training data (thus optimization is successful), but large on (future) test data, indicating that the training process has converged to a function that is far from being an optimal representation.

We can gain insight about the possible causes of this phenomenon by looking at the following decomposition of the cross-entropy \cite{achille2017emergence}:
\label{prop:generalization-factorization}
\begin{multline}
H_{p,q}(\y^t|\x^t,w)
= H_p(\y^t|\x^t,\theta) + I(\theta;\y^t|\x^t, w) +  
\KL{q(\y^t|\x^t,w)}{p(\y^t|\x^t,w)} - I(\y^t;w|\x^t, \theta),
\label{eq:decomposition}
\end{multline}
where $w \sim q(w|\x^t,\y^t)$.
The first term of the right-hand side of \eqref{eq:decomposition} relates to the intrinsic error and depends only on $p_\theta$; the second term measures how much of the information past data contain about the parameter $\theta$ is captured by the weights; the third term relates to the efficiency of the model and the class of functions $f_w$ with respect to which the loss is optimized. The last, and only negative, term relates to how much information about the labels is memorized in the weights, regardless of the underlying data distribution. Absent any intervention, the left-hand side (LHS) of \eqref{eq:decomposition} can be minimized by just maximizing the last term, {\em i.e.,} by memorizing the dataset, which amounts to overfitting and yields poor generalization.
Traditional machine learning practice suggests that this problem can be avoided by reducing the complexity of the model, or by regularizing its parameters. However, it has been shown \cite{zhang2016understanding} that common architectures and regularization methods always allow the network to  memorize a dataset.

Memorization can however be prevented by adding the last term back to the loss function, leading to a regularized loss $H_{p,q}(\y | \x, w) + I(\y;w|\x,\theta)$, where the negative term on the RHS is canceled. However, computing, or even approximating, the value of $I(\y,w|\x,\theta)$ is at least as difficult as fitting the model itself.

To overcome this problem, consider $\D = (\x^t,\y^t)$, the collection of all past data that we are using to infer the model parameters $w$. Notice that to successfully learn the distribution $p_\theta$, we only need to memorize in $w$ the information about the latent parameters $\theta$, that is we need $I(\D;w) = I(\D;\theta) \leq H(\theta)$, which is bounded above by a constant. On the other hand, to overfit, the term
$I(\y;w|\x) \leq I(\D;w|\theta)$ needs to grow linearly with the number of training samples $N$. We can exploit this fact to prevent overfitting by adding a Lagrange multiplier $\beta$ to make the amount of information constant with respect to $N$, leading to the regularized loss function
\begin{equation}
\label{eq:variational-regularizer}
\L(p(w|\D)) = H_{p,q}(\y|\x,w) + \beta I(w;\D),
\end{equation}
which is, remarkably, the same IB Lagrangian in \eqref{eq:IBL}, but now interpreted as a function of $w$ rather than $\z$. Under appropriate assumptions on the form of the posterior $q(w|\D)$, the term $I(w;\D)$ can be computed in closed form, and we can optimize \cref{eq:variational-regularizer} efficiently \cite{kingma2015variational, achille2017emergence}.

Thus, as we have seen, the IB Lagrangian emerges as a natural criterion \emph{both} for inferring
a representation of the test datum $\x$ that is sufficient and
invariant (with no explicit notion of overfitting), and for inferring
a representation $w$ of the training dataset (past data) $\D$ that avoids overfitting (with no explicit notion of invariance). A natural question, which we will address in \Cref{sec:duality}, is if, and how, these two representations, and their corresponding IB Lagrangians, are related to each other.

\begin{rmk}
An alternative approach to the generalization problem is to use a Bayesian framework, and to find a posterior distribution $q(w|\D)$ over the weights that maximizes the marginal log-likelihood $q(\y|\x) = \int q(\y|\x,w)  d q(w|\D)$ of the data, subject to a given prior $p(w)$. Rather than optimizing $q(\y|\x)$ directly, which would be computationally expensive, one can maximize the Variational Lower Bound (VLBO)
\[
\log p(\y^t|\x^t) \geq H_{p,q}(\y^t|\x^t,w) + \KL{q(w|\D)}{p(w)}.
\]
The IB Lagrangian \cref{eq:variational-regularizer} can be seen as a generalization of Bayesian learning, where we have increased flexibility in selecting the regularizer by the added multiplier $\beta$.
\end{rmk}

\subsection{Information, generalization, and flat minima}
\label{sec:flatness}

Thus far we have suggested that adding the explicit information regularizer $I(w; \D)$ prevents the network from memorizing the dataset and thus avoids overfitting, which is also confirmed empirically in \cite{achille2017emergence}.
However, common networks are not commonly trained with this information regularizer, thus seemingly undermining the theory.
However, even when not explicitly controlled, $I(w;\D)$ is implicitly regularized by the use of SGD.
In particular, empirical evidence suggests that SGD biases the optimization toward ``flat minima'': local minima whose Hessian has mostly small eigenvalues \cite{dinh2017sharp}.
These minima can be interpreted exactly as having low information $I(w;\D)$, as  suggested early on by \cite{hochreiter1997flat}. As a consequence of previous claims, flat minima can then be seen as having better generalization properties.

More precisely, let $\hat w$ be a local minimum of the cross-entropy loss $H_{p,q}(\y|\x,w)$, and let $\H$ be the Hessian at that point.
Then, under suitable assumptions on the form of the posterior, for the optimal choice of the posterior parameters we have \cite{achille2017emergence}:
\[
I(w;\D) \leq \half K [\log \norm{\hat{w}}_2^2 + \log \norm{\H}_* - K\log (K^2 \beta/2)],
\]
where $K = \dim(w)$ and $\|\H\|_* = \mathrm{tr}(\H)$ denotes the nuclear norm of the matrix. Therefore the information in the weights is upper-bounded by the nuclear norm (and hence the ``flatness'') of the Hessian. Notice that a converse inequality, that is, low information implies flatness,
needs not hold, so sharp minima can in principle generalize well, as  proved by \cite{dinh2017sharp}.

In the next section we show that the quantity of information on the weights is connected not only to the geometry of the loss function, but also to the minimality (invariance) and disentanglement of the activations.
In particular, this shows that weight regularization, whether implicit (SGD) or explicit (IB Lagrangian), biases the optimization towards good representations. 

\subsection{Duality of the representations}
The core link between information in the weights, and hence flatness of the local minima, minimality of the representation, and disentanglement can be described by the following proposition from \cite{achille2017emergence}:

\begin{prop}
Let $\z=W\x$ be a single layer of a network. Under opportune hypotheses on the form of $q(W|\D)$, we can find a strictly increasing function $g(\alpha)$ s.t. we have the uniform bound
\[
g(\alpha) \leq \frac{I(\x;\z) + TC(\z)}{\dim(\z)} \leq g(\alpha) + c,
\]
where $c=O(1/\dim(\x))\leq 1$, 
and $\alpha$ is related to $I(w;\D)$ by  $\alpha = \exp\set{-I(W;\D)/\dim(W)}$.
In particular, $I(\x;\z) + TC(\z)$ is tightly bounded by $I(W;\D)$ and increases strictly with it.
\end{prop}

Using the Markov property of the layers, we can now easily extend this bound to multiple layers.
Let $W^k$ for $k=1,...,L$ be weight matrices,
and let $\z_{i+1} = \phi(W^k \z_k)$,
where $\z_0 = \x$ and $\phi$ is any nonlinearity. Under the same assumptions as the previous result, one can prove
\[
I(\z_L;\x) \leq \min_{k<L} \set{\dim(\z_k) \big[g(\alpha^k) + 1 \big] }
\]
where $\alpha^k = \exp\set{-I(W^k;\D)/\dim(W^k)}$.

Together with the results of \Cref{sec:core-results}, this implies that regularized networks containing low information in the weights, automatically learn a representation of the input which is both more invariant to nuisances and more disentangled. Moreover, by \Cref{sec:flatness}, SGD is biased toward such representations.

This result is important because it establishes connections between the weights, which are a representation of {\em past data}, given and used to optimize a loss function that knows nothing about sufficiency, minimality, invariance and disentanglement, and representation of {\em future data}, that emerge to have precisely those properties. Such connections are peculiar to the class of functions implemented by deep neural networks and do not apply to any generic function class.

Finally, we have all the elements to extend the notion of representation, and the optimization involved in inferring it (which encompasses system identification and filtering) to a dynamic setting.

\section{Representing time series}
\label{sec:recurrent}

In this section we consider the case where the data are not drawn i.i.d. from a distribution with constant underlying parameters. Instead, we assume that the representation can evolve over time according to a probability law that does not. More details can be found in \cite{achille2017deep}.

\subsection{Hidden state dynamic model}

Many standard models for filtering and control assume the existence of a hidden state $\z_t$ which evolves following a Markov process through some state transition probability $p(\z_{t+1}|\z_t,\u_t)$, where we made the dependency on the control action $\u_t$ explicit.
The observations $\x_t$ are sampled from the hidden state $\z_t$ with some distribution $p(\x_t|\z_t)$, as described by the following graphical model:

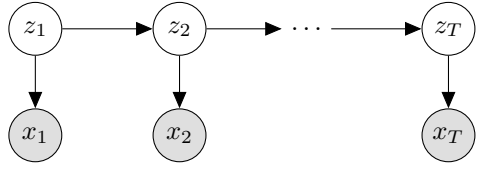
\begin{center}
\begin{tikzpicture}[y=.7cm, x=1.2cm]
    \zxfig{1}
    \zxfig{2}
    \twodotsfig{3}
    \zxfig[T]{4}
\end{tikzpicture}
\end{center}

The fundamental assumption of this model is that there is a random variable of bounded complexity, the state $\z_t$, that \emph{separates} new observations
$\x_{t+1}$ from all past ones $\x^t$. 
The advantage of having such a variable is apparent in the classical   filtering equations:
\begin{align}
\label{eq:filtering}
p(\z_{t+1} | \x^{t+1}, \u^t) &\propto p(\x_{t+1} | \z_{t+1}) \int p(\z_{t+1} | \z_t,\u_{t}) p(\z_t | \x^t, \u^{t-1}) d\z_t \\
p(\x_{t+1}|\x^t, \u^t) &= \int p(\x_{t+1}|\z_{t+1}) p(\z_{t+1}|\x^t, \u^t) d \z_{t+1}
\end{align}
Here, all the information about the past data $\x^t$ is contained in the (relatively small) posterior $p(\z_t|\x^t)$. In this sense, the posterior is \emph{sufficient} for the state update, \textit{i.e.}, for computing $p(\z_{t+1} | \x^{t+1}, \u^{t})$, and for the prediction of the data, \textit{i.e.}, computing $p(\x_{t+1}|\x^t)$.

In a hidden Markov model or Kalman filter, the transitions are assumed to be linear, and the state and observations either Gaussian or discrete. In these cases, the posterior can updated easily and there are efficient algorithms to infer the model parameters of the system. However, for many real problems the integrals in the filtering equation are not tractable since the transition operator is often non-linear. In this case the complexity of updating the posterior may grow exponentially \cite{bar-shalomF87}. Furthermore, the data generating distribution $p(\x_{t+1}|\z_{t+1})$ is difficult to compute, or even to approximate. Finally, while we can always artificially ignore long-term dependencies and consider the system Markovian by augmenting the state $X_t=[\z_{t-k},\ldots, \z_t]$, the resulting state may be too complex to handle.

While there is no obvious solution to these problems in general, it is often the case that we are not interested in predicting the data, but just the control action $\y$ which can be quite different and far smaller than the data. In the next section we see that this can guide the design of efficient filters.

\subsection{Separating representation}

Rather than explicitly looking for a Markovian state that can generate the observed data $\x_t$, \textit{i.e.}, inferring representations for \emph{prediction of the data}, we focus on finding a representation (proxy state) $\z_t$ to predict a \emph{task variable} $\y_t$, for instance a control input, which is generally far lower-dimensional than the data, and that allows causal and recursive posterior update using only the latest measurements. In this sense, this section is about inferring representations for \emph{control.} 

Motivated by \cref{eq:filtering}, we define the variable $\z_t$ through its posterior distribution $q(\z_t|\x^t,\u^t)$,%
\footnote{We use $q$ to distinguish the (unknown) data distribution $p$ from our model distribution.}
and we require that it satisfies the following:
\begin{enumerate}[label={(\arabic*)}]
\item \textbf{prediction}: the posterior of $\z_t$ is sufficient of $\x^t$ and $\u^t$ for $\y_{t+k}$, that is, for each $0\leq k<n$, we have 
\[p(\y_{t+k}|\x^t,\u^{t+k-1}) = \int q(\y_{t+k}|\z_t,\u_t^{t+k})q(\z_t|\x^t,\u^{t-1})d\z_t\]
\item \textbf{update}: the posterior of $\z_t$ is sufficient of $\x^t$ and $\u^t$ for $\z_{t+1}$
\[
q(\z_{t+1}|\x^{t+1},\u^t) = \int q(\z_{t+1}|\z_t, \x_{t+1}, \u_t) q(\z_t|\x^t,\u^{t-1})d\z_t.
\]
\end{enumerate}
Note that, like the classical filtering equations, this density propagation is exact. However, unlike the filtering equations, we can directly learn the transition probability $q(\z_{t+1}|\z_t, \x_{t+1}, \u_t)$ rather than use Bayes' rule, and therefore there is no need to compute the posterior $p(\x_{t+1}|\z_{t+1})$, which is generally intractable for high-dimensional and complex data such as natural images. The separator, in this case, is not the random variable $\z_t$, but the posterior density $p(\z_t | \x^t, u^t)$, which is in general infinite-dimensional. This model allows us to explicitly modulate complexity of the representation with the fidelity of the separation.

\begin{example}[Kalman filter]
  The method we propose better reduce, in the linear Gaussian case, to the Kalman filter. Indeed, it does so in two different ways. First, let $\z_t$ be the state of a linear time-invariant Gaussian state-space model, and let the task be one-step prediction, $\y_t=\x_{t+1}$. Now, let $\hz_t$ be a random variable such that $q(\hz_t|\x^t) = p(\z_t|\x^t)$, where $p(\hz_t|\x^t)$ is the posterior computed by the Kalman filter, and let $q(\y_t|\hz_t) = p(\x_{t+1}|\z_t)$. Then, trivially, $p(\y_t|\x^t) = p(\x_{t+1}|\x^t) = \int p(\x_{t+1}|\z_t) p(\z_t|\x^t)d\z_t = \int q(\x_{t+1}|\hz_t) q(\hz_t|\x^t)d\z_t$, so
the posterior computed by the Kalman filter is sufficient for predicting future measurements.
Moreover, by letting $q(\hz_{t+1}|\hz_t, \x_{t+1}) = p(\z_{t+1}|\z_t, \x_{t+1}) = \frac{1}{Z} p(\x_{t+1}|\z_{t+1}) p(\z_{t+1}|\z_t)$, we see that $q(\hz_t|\x^t)$ is also sufficient for the update. Therefore the posterior computed by the Kalman filter satisfies both the prediction (1) and update model (2).
Notice, however, that this is not the only option. Instead, let $\hz_t=(\hat{\x}(\x^t),P(\x^t))$ be the mean and covariance of the innovation computed by the Kalman filter. Then $\z_t$ is  a deterministic sufficient statistic  (function of the past $\x^t$). Notice that, in this case, the dimension of the representation $\hz_t$ is larger and the update equation is given by the more complex Riccati equation. So, by adopting a deterministic representation, we have had to increase its computational complexity. 
\end{example}

While in \cref{eq:filtering} we need to use the prediction probability $p(\x_t|\z_t)$ to update the posterior, which is not tractable when the data $\x_t$ is high-dimensional, using conditions (1) and (2) we have the simple iterative update rules:
\begin{align*}
q(\y_t|\x^t,\u^t) &= \int q(\y_t|\z_t) q(\z_t|\x^t, \u^t) d \z_t, \\
q(\z_{t+1}|\z_t, \x^{t+1},\u^{t+1}) &= \int q(\z_{t+1}|\z_t, \x_t,\u_t) \, q(\z_t|\x^t,\u^t) d\z_t.
\end{align*}
Unlike \cref{eq:filtering} these update equations only involve distributions over $\y_t$ and $\z_t$, which are assumed to have lower effective dimension than the data $\x_t$, or at least have a simpler distribution (\textit{i.e.}, discrete, Gaussian).

Moreover, notice that if we restrict $q(\z_t|\x^t,\u^t)$ to be degenerate (\textit{i.e.}, a Dirac delta), so $\z_t$ is a deterministic function of the past history of the measurements, which this framework allows, then the integrals are trivial and all updates can be computed exactly. On the other hand, allowing a more complex form for $q(\z_t|\x^t,\u^t)$ could drastically simplify the computation of both $q(\y_t|\z_t)$ and $q(\z_{t+1}|\x_{t+1},\u_t,\z_t)$, so there is a trade-off between the cost of computing the integrals for $q(\z_t|\x^t,\u^t)$ and the complexity of the prediction and update rules, as seen in the case of the Kalman filter. More specifically, when the model is linear and the driving input white, zero-mean Gaussian and i.i.d., the posterior is Gaussian. Thus one can consider both the posterior itself, or the parameters that represent it (mean and covariance matrix), as the separator, with the latter being a deterministic representation.

While the complexity of a posterior $q(\z_t|\x^t,\u^{t-1})$ sufficient for the task $\y_t$ is generally much smaller than of what would be required to predict $\x_{t+1}$, a representation $\z_t$ that satisfies all the required properties may still have an high dimension or high complexity. What we are after is an explicit way to trade off complexity with quality of the representation, represented by its ``degree of sufficiency and Markovianity.''
As we have already seen in the static case, this trade-off can be expressed by the IB Lagrangian, which is now
\[
\L = \frac{1}{T}\sum_{t=1}^T \sum_{k=0}^n H_{p,q}(\y_{t+k}|\x^t,\u^{t+k-1}) + \beta I(\z_t;\x^t,\u^t).
\]
Where $H_{p,q}$ is the cross entropy between the real data distribution $p(\y_t|\x^t,\u^t)$ and our model distribution $q(\y_t|\x^t,\u^t)=\int q(\y_t|\z_t)q(\z_t|\x_t,\u_t)d\z_t$ previously defined.

\begin{prop}[$n$-step prediction loss]
Given $q(\y_t|\z^t,\u^t)$ and $q(\z_t|\x_t,\u_t, \z_{t-1})$, define $q(\y_t|\x^t)=\int q(\y_t|\z_t) q(\z_t|\x^t, \u^t) d\z_t$ as above. Then the cross-entropy loss
\[
\L = \frac{1}{T}\sum_{t=1}^T \sum_{k=0}^n H_{p,q}(\y_{t+k}|\x^t,\u^{t+k-1})
\]
is minimized if and only if the posterior $q(\z_t|\x^t)$ of $\z_t$ {\em separates} $\y_t$ from the past data $\x^t$, meaning that $p(\y_t|\x^t,\u^t) = \int q(\y_t | \z_t,\u^t) q(\z_t|\x^t,\u^t) d\z_t = F(q(\z_t|\x^t,\u^t))$ for almost all $\x^t$. 
\end{prop}
\begin{proof}
To simplify the notation, we only consider the case $n=0$ (that corresponds to smoothing), the general case being identical. Recall that $H_{p,q}(\y_t|\x^t) = H_p(\y_t|\x^t) + \E_{\x^t\sim p(\x^t)} \KL{p(\y_t|\x^t)}{q(\y_t|\x^t)}$, so
\begin{align*}
\L &= \frac{1}{T}\sum_{t=1}^T H_{p,q}(\y_t|\x^t) \\
&= \frac{1}{T} \sum_{t=1}^T H_p(\y_t|\x^t) + \frac{1}{T}\sum_{t=1}^T \E_{\x^t} \KL{p(\y_t|\x^t)}{q(\y_t|\x^t)} \\
&\geq \frac{1}{T} \sum_{t=1}^T H_p(\y_t|\x^t).
\end{align*}
Since the degenerate representation $\z_t = \x^t$ trivially reaches the lower bound, for any representation minimizing the loss function we must have
$\E_{\x^t} \KL{p(\y_t|\x^t)}{q(\y_t|\x^t)} = 0$ for all $t$ and a.e.\ $\x^t$. In particular, for a.e.\ $\x^t$ we have $p(\y_t|\x^t) = q(\y_t|\x^t) = \int q(\y_t|\z_t) q(\z_t|\x^t) d\z_t$.
\end{proof}

\begin{rmk}
Notice that it is not the random variable $\z_t$ that separates $\y_t$ from $\x^t$, \textit{i.e.} $\y_t \perp \x^t | \z_t$, as it was in the static case. Instead, it is its (posterior) distribution $q(\z_t|\x^t)$ that acts as the separator. However, if the latter is finitely-parametrized, $q(\z_t|\x^t) = q_{\phi(\x_t)}(\z_t)$ where $q_\phi$ is a parametrized family of probability distributions and $\phi(\x^t)$ is a function, then $\y_t \perp \x^t | \phi(\x^t)$, \textit{i.e.,} the parameters of the distributions can be interpreted as a finite-dimensional representation that separates the past data from the task.
\end{rmk}

\begin{cor}
Suppose that there exist a separating variable $\z_t$ of finite complexity. Then, in the limit $\beta\to 0$, the IB Lagrangian recovers a separating representation.
\end{cor}
\begin{proof}
Since $\z_t$ has finite complexity, $\beta I(\z_t;\x^t) \to 0$ as $\beta \to 0$. Therefore, in the limit the minimum of the Lagrangian is exactly 
\[
\L' = \frac{1}{T}\sum_{t=1}^T \sum_{k=0}^n H_{p,q}(\y_{t+k}|\x^t,\u^{t+k-1}).
\]
Therefore any other minimizer $\z'$ of the IB Lagrangian must also minimize $\L'$, and, by the previous proposition, it must be a separating distribution.
\end{proof}

\section{A separation principle for control}
\label{sec:separation}

In the previous section we have seen that, given a task $\y$ (a random variable to predict) it is possible to infer a representation $\z$ that trades off complexity with sufficiency and Markovianity. We now specialize this program for a \emph{control} task, so that a controller operating on the representation behaves as if it had access to the entire past history of the data, analogously with the separation principle in linear-quadratic Gaussian (LQG) optimal control. Unlike LQG, however, in general there is no finite-dimensional sufficient statistic, and therefore, following the program above, we seek for a representation that trades off of complexity with fidelity.

To this end, assume that our control task consists of minimizing a control loss $R$ such that
\[
R = \sum_{t=1}^T r_t(\x^t,\u^t),
\]
where $r_t = r(\x^t,\u^t)$ is a possibly stochastic function of the true (global) state of the system $\z_t$ and the actions. Notice that even if the system is not Markovian, we can always assume such a global state exists, in the worse case $\z_t=\x^t$. Notice that LQG and other standard control losses can be written in this form. To simplify, we consider a finite horizon $T<\infty$.

We claim that if the posterior of $\z_t$ is a sufficient representation of the data $\x^t$ for the task $\y_t=r_t$, then there exists an optimal control policy $\pi'$ which is a function of the posterior $q(\z_t|\x^t,\u^t)$ alone.

\begin{prop}
\label{prop:rl}
Let $\z_t$ be such that the posterior $q_t = q(\z_t|\x^t,\u^t)$ of $\z_t$ is sufficient of $\x^t,\u^t$ for $r_t$, meaning that
\[
p(r_{t+k}|\x^t,\u^{t+k}) = \int q(r_{t+k}|\z_t, \u_t^{t+k}) q(\z_t|\x^t,\u^t) d\z_t.
\]
Then, there exist an optimal control policy $\u_{t+1}=\pi^*(q_t)$ that minimizes the expected risk $\E[R|\pi^*]$ and depends on the past data $\x^t, \u^t$ only through $q_t$.
\end{prop}
\begin{proof}
Adopting standard reinforcement learning notation, let $Q^\pi_{>t}(\x^t,\u^t,\u) = \E[R_{>t}|\x^t,u_{t+1}=u,\pi]$ be the expected value of $R_{>t} = \sum_{t'=t}^T r_{t'}$ when following the policy $\pi$ for the last $T-t$ steps given the observation history and action history $\x^t,\u^t$ until now. Define the optimal Q-function $Q^*(\x^t,\u^t,\u) = \max_\pi Q^\pi_{>t}(\x^t,\u^t,\u)$.

Recall that, given $Q^*(\x^t,\u^t,\u)$, the optimal policy is given by 
$\pi^*(\x^t,\u^t) = \argmax_\u Q^*_{>t}(\x^t,\u^t,\u)$. Therefore, to prove that the optimal policy depends only on $q_t$, it suffices to prove that $Q^*_{>t}(\x^t,\u^t,\u) = Q(q_t,\u)$, i.e. that we can compute the optimal Q-function given $q_t$ alone instead of the whole history $\x^t,\u^t$.
This follows trivially from the fact that
\begin{align*}
Q^*_{>t}(\x^t,\u^t,\u) &= \min_{\u_t^{T} : \u_{t+1} = \u} \sum_k^{T-t} \E[r_{t+k}|\x^t, \u^{t+k}] \\
&= \min_{\u_{t+1}^{T} : \u_t = \u} \int \set{\sum_{k=0}^{T-t} r_{t+k}\, q(r_{t+k}|\z_t,\u_t^{t+k})}\, d q(\z_t|\x^t,\u^t).
\end{align*}
\end{proof}

\begin{rmk}
Notice that this proposition does not gives an explicit way of learning a policy (since a na\"ive application would require a brute force optimization over all possible actions).
Rather, the  usefulness of the theorem is in that it proves that any representation sufficient to predict the rewards $r_t$, a fairly general condition, is also sufficient for control. In particular, if $r_t=r_t(m_t)$ is a function of some (low-dimensional) measurement $m_t$, as \cite{dosovitskiy2016learning} suggests, a representation $\z_t$ trained to predict those measurements $m_t$ will also be sufficient for $r_t$ and hence for control. This fact is implicitly exploited in \cite{dosovitskiy2016learning} to learn a state of the art control policy for complex task and high-dimensional data.
\end{rmk}

\section{Discussion}
\label{sec:discussion}

We have framed the problem of system identification for the purpose of control as that of inferring {\em not} deterministic statistics of sufficiently-exciting time series, but rather of an approximation of the posterior of the control loss  given past measurements. While this is in general infinite-dimensional, universally-approximating function classes, such as neural networks, can be employed in the inference. This yields some nice properties relative to classical Bayesian filtering: First, we do not need to apply Bayes' rule, and therefore there is no partition function to compute, to the benefit of computational complexity. Second, we do not need to make a strict assumption of Markovianity. Instead, we can explicitly trade off complexity with fidelity of the approximation of the posterior. Such a posterior is the separator that plays the analogous role of the state of a Gaussian linear model in classical linear identification.

The good news is that the representations learned by generic stochastic gradient descent, while being agnostic of desirable properties of the resulting representation, end up enforcing them through implicit regularization, as we show for the static case.

Now for a few caveats. First, the representations we aim to infer are optimal when they are as good as the data for the chosen task. This does not mean they are good: If the data is uninformative (or non-sufficiently exciting), there is no guarantee that can be made on the quality of the representation, other than it is sufficient, meaning as good as the data (it can be no more, per the Data Processing Inequality). A completely independent problem is how to get as exciting data as possible, which is the problem of Active Learning or Experiment Design, that can be framed as an optimal control problem, which we do not address here.

Second, we are not suggesting that the model we propose is tractable in its most general form, or that training a neural network to minimize the IBL proposed is easy. However, we show that  minimizing a simple cross entropy for a particular task (the control loss) leads to a representation which is sufficient for control.  One should notice that this approach has strong links with \cite{dosovitskiy2016learning}, but also with reinforcement learning. Indeed, both can be seen as ways of making the algorithm tractable by directly approximating the expected loss for a given action.

More importantly, this class of tools opens a number of potentially exciting research avenue, both applied -- making use of the power of these representations, and implementing efficient algorithms to infer them -- and theoretical, as little is known about the properties of these representation and their approximation bounds. This approach promises to re-open a field that has been shackled between the linear case, which is nice and elegant and for which a plethora of results are known, but that is very limited applicability, and the general case, where there is little to say, and little that works in practice.

\bibliographystyle{plain}
\bibliography{bibliography}

\end{document}